\documentclass{article}
\usepackage[accepted]{icml2021} 
\usepackage[table]{xcolor}
\usepackage{tikz}
\usetikzlibrary{matrix}
\usepackage{booktabs}
\usepackage{todonotes}
\presetkeys{todonotes}{inline}{}

\usepackage{subfig}
\usepackage{mathtools}
\usepackage{amssymb,amsfonts,dsfont,amsthm}
\usepackage{amsmath}
\usepackage[ruled,vlined, algo2e]{algorithm2e}
\newtheorem{theorem}{Theorem}
\newtheorem{corollary}{Corollary}
\newtheorem{definition}{Definition}

\newcommand{\votepurple}[1]{\textcolor{Purple}{$\bigstar$}}
\newcommand{\voteyellow}[1]{\textcolor{Goldenrod}{$\bigstar$}}
\newcommand{\voteblue}[1]{\textcolor{RoyalBlue}{$\bigstar$}}
\newcommand{\votepink}[1]{\textcolor{Pink}{$\bigstar$}}

\newcommand{\insts}[0]{\mathcal{I}}
\newcommand{\inst}[0]{i}

\newcommand{\policy}[0]{\pi}

\newcommand{\stateRL}[0]{s}

\RestyleAlgo{algoruled}
\DontPrintSemicolon
\LinesNumbered
\SetAlgoVlined
\SetFuncSty{textsc}

\SetKwInOut{Input}{Input}
\SetKwInOut{Output}{Output}
\SetKw{Return}{return}

\newcommand{\citeN}[1]{\citeauthor{#1}~(\citeyear{#1})}

\DeclareMathOperator*{\argmax}{arg\,max}

\usepackage{pifont}

\usepackage{xspace} 
\newcommand{\spc}{{\textsc{SPaCE}}\xspace}

\usepackage{xstring}
\sfcode`\.=1001
\sfcode`\?=1001
\sfcode`\!=1001
\sfcode`\:=1001
\newcommand{\autocase}[1]{\ifnum\ifhmode\spacefactor\else2000\fi>1000 \uppercase{#1}\else#1\fi}

\newcommand{\task}{\autocase{i}nstance\xspace}
\newcommand{\atask}{\autocase{a}n \task}
\newcommand{\tasks}{\autocase{i}nstances\xspace}

\usepackage[switch]{lineno}

\usepackage{xr-hyper}
\makeatletter
\newcommand*{\addFileDependency}[1]{
  \typeout{(#1)}
  \@addtofilelist{#1}
  \IfFileExists{#1}{}{\typeout{No file #1.}}
}
\makeatother

\newcommand*{\myexternaldocument}[1]{%
    \externaldocument{#1}%
    \addFileDependency{#1.tex}%
    \addFileDependency{#1.aux}%
}
\myexternaldocument{appendix}

\usepackage{microtype}
\usepackage{graphicx}
\usepackage{booktabs} 

\usepackage{hyperref}

\icmltitlerunning{SPaCE}

\begin{document}

\twocolumn[
\icmltitle{Self-Paced Context Evaluation for Contextual Reinforcement Learning}



\icmlsetsymbol{equal}{*}

\begin{icmlauthorlist}
\icmlauthor{Theresa Eimer}{luh}
\icmlauthor{André Biedenkapp}{fr}
\icmlauthor{Frank Hutter}{fr,bo}
\icmlauthor{Marius Lindauer}{luh}
\end{icmlauthorlist}

\icmlaffiliation{luh}{Information Processing Institute (tnt),  Leibniz University Hannover, Germany}
\icmlaffiliation{fr}{Department  of  Computer  Science,  University  of  Freiburg, Germany}
\icmlaffiliation{bo}{Bosch Center for Artificial Intelligence, Renningen, Germany}

\icmlcorrespondingauthor{Theresa Eimer}{eimer@tnt.uni-hannover.de}

\icmlkeywords{Reinforcement Learning, Curriculum Learning}

\vskip 0.3in
]



\printAffiliationsAndNotice{}
\begin{abstract}
Reinforcement learning (RL) has made a lot of advances for solving a single problem in a given environment; but learning policies that generalize to unseen variations of a problem remains challenging.
To improve sample efficiency for learning on such \emph{\tasks} of a problem domain, we present \emph{Self-Paced Context Evaluation (\spc)}.
Based on self-paced learning, \spc automatically generates \task curricula online with little computational overhead. 
To this end, \spc leverages information contained in state values during training to accelerate and improve training performance as well as generalization capabilities to new \tasks from the same problem domain. Nevertheless, \spc is independent of the problem domain at hand and can be applied on top of any RL agent with state-value function approximation. 
We demonstrate \spc's ability to speed up learning of different value-based RL agents on two environments, showing 
better generalization capabilities and
up to $10\times$ faster learning compared to naive approaches such as round robin or SPDRL, as the closest state-of-the-art approach.
\end{abstract}

\section{Introduction}
Although Reinforcement Learning (RL) has performed impressively in settings like continuous control~\cite{lillicrap-iclr16}, robotics~\cite{openai-corr19a} and game playing~\cite{silver-nature16a,vinyals-nat19a}, their applicability is often very limited.
RL training on a given task takes a lot of training samples, but the skills acquired do not necessarily transfer to similar tasks as they do for humans.
An agent that is able to generalize across variations of a task, however, can be applied more flexibly and has a lower chance of succeeding when presented with unseen inputs.
Therefore improving generalization means improving sample efficiency and robustness to unknown situations. We view these as important qualities for real-world RL applications.

Curriculum learning \cite{bengio-icml09} aims to bridge the gap between agent and human transfer capabilities by training an agent the same way a human would learn: transferring experience from easy to hard variations of the same task.
It has been shown that generating such \emph{\tasks} with increasing difficulty to form a training curriculum can improve training as well generalization performance \cite{dendorfer-accv20-goalgan, matiisen-corr17, zhang-neurips20-vds}.
As information about instance difficulty is often not readily available, many approaches utilize the agent's progress markers, such as evaluation performance, confidence in its policy or its value function to minimize the need for domain knowledge \cite{wang-gecco19-poet, klink-neurips20-spdrl}. 
Because the progression is dictated by the agent's learning progress, this is called Self-Paced Learning \cite{kumar-neurips10a}.

\tasks in a curriculum can vary from the core task in different aspects, such as varying goals or movement speeds (see Fig.~\ref{fig:example_instances}).
While only selecting different goals states as \tasks is common for curriculum learning methods \cite{dendorfer-accv20-goalgan, zhang-neurips20-vds}, changing transition dynamics are important considerations regarding the robustness of a policy.
A dynamic change in robotics could for example be caused by a broken joint that the agent now has to adapt to.
To allow these changes in the transition dynamics, in addition to goal changes in the \tasks, we consider \emph{contextual} RL instead.
\begin{figure}[tbp]
    \begin{minipage}[t]{0.12\textwidth}
    \centering
    \scalebox{0.75}{
    \begin{tikzpicture}
        \draw [fill=blue, fill opacity=0.01] (0,0) -- (2,0) -- (2,2) -- (0,2) -- (0,0);
        \node[circle,fill=yellow!20!white,draw=gray,thin, minimum size=0.1cm] at (.25, .25){};
        \node[circle,fill=green!20!white,draw=gray,thin, minimum size=0.1cm] at (.75, .25){};
    \end{tikzpicture}
    }
    \end{minipage}%
    \begin{minipage}[t]{0.12\textwidth}
    \centering
    \scalebox{0.75}{
    \begin{tikzpicture}
        \draw [fill=blue, fill opacity=0.05] (0,0) -- (2,0) -- (2,2) -- (0,2) -- (0,0);
        \node[circle,fill=yellow!20!white,draw=gray,thin, minimum size=0.1cm] at (.25, .25){};
        \node[circle,fill=green!20!white,draw=gray,thin, minimum size=1cm] at (1.25, 1.25){};
    \end{tikzpicture}
    }
    \end{minipage}
    \centering
    \begin{minipage}[t]{0.12\textwidth}
    \centering
    \scalebox{0.75}{
    \begin{tikzpicture}
        \draw [fill=blue, fill opacity=0.13] (0,0) -- (2,0) -- (2,2) -- (0,2) -- (0,0);
        \node[circle,fill=yellow!20!white,draw=gray,thin, minimum size=0.1cm] at (.25, .25){};
        \node[circle,fill=green!20!white,draw=gray,thin, minimum size=0.6cm] at (1, 1.5){};
    \end{tikzpicture}
    }
    \end{minipage}%
    \caption{Example \tasks of the contextual PointMass environment.
    The agent's yellow starting point, the green goal and floor friction (indicated by shading) are part of the context and vary between \tasks.}
    \label{fig:example_instances}
\end{figure}
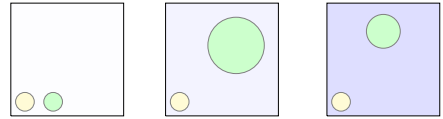%

Our contributions are as follows:
\begin{enumerate}
    \item We propose \spc, a new self-paced learning algorithm,
to automatically generate \task curricula in a general contextual RL setting, without any knowledge about instance difficulty being required and  with access to only a limited set of \tasks (see Section~\ref{space}).
    \item We show the convergence behavior of \spc to be at least as good as round robin (see Section~\ref{sub:convergence}).
    \item We demonstrate that \spc is capable of outperforming a round robin baseline~\cite{speck-prl20} as well as similar self-paced methods (see Section~\ref{experiments}).
\end{enumerate}

\section{Related Work}

There are different approaches to increase generalization capability in RL. Their goals and scopes differ substantially, however. 
MAML \cite{finn-icml17} and related meta-RL methods pre-train an agent such that specializing on one of the training tasks is then very efficient. These take different approaches of aggregating and propagating the gradients in training and are complementary approaches to \spc.

Domain randomization \citep[DR;][]{tobin-iros17} on the other hand varies the task space. In essence, DR creates new instances of tasks  in order to force the agent to adapt to alterations in its observations and policy. 
Other examples such as POET \cite{wang-gecco19-poet} and ADR \cite{openai-corr19a} sample instances at random but order them by leveraging knowledge about the environment.
Without prior knowledge of the target distribution, however, making appropriate changes is hard, resulting in either too little variation to facilitate generalization or deviating so much that the problem becomes too hard to learn.
Other approaches utilize human expert knowledge to facilitate generalization performance, such as human-in-the-loop RL \cite{thomaz-aaai06} or imitation learning \cite{hussein-acm17}.

Curriculum learning~\cite{bengio-icml09} uses expert knowledge to generate an ordering of training instances in such a way that knowledge can be transferred from hard to easy instances.
There are different approaches for automatically generating such instance curricula, including learning how to generate training instances \cite{dendorfer-accv20-goalgan, such-icml20a} similar to a teacher \cite{matiisen-corr17, turchetta-neurips20} or leveraging self-play as a form of curriculum generation \cite{sukhbaatar-iclr18, silva-bracis19}. 
In most of these cases, some knowledge of the \task space is required in order to either define a measure of \task difficulty or how to generate new instances. 
While instance generation requires only little prior knowledge, a separate agent will need to learn to generate instances of appropriate difficulty, which increases the training overhead significantly.

Value Disagreement based Sampling \citep[VDS;][]{zhang-neurips20-vds} on the other hand builds curricula for goal-directed RL. VDS uses the disagreement between different agents trained on the same instances to measure which training instance should be trained on next. Like its building block HER~\cite{andrychowicz-neurips17-her}, VDS is only compatible with goal-directed off-policy RL. 

One approach to order the training instances is explicitly using an agent's performance as an ordering criterion instead, called Self-Paced Learning. 
This can be done using the agent's value function as a substitute for actual episode evaluations. SPDRL \cite{klink-neurips20-spdrl} uses this idea to generate new instances uniquely suited to the agent's current training progress in order to progress towards specific hard instances. 
While this eliminates the need for a teacher, researchers instead need to know a priori which instances are considered the hard target instances and where the agent should start training in relation to them.

\section{Contextual Reinforcement Learning}
Before we describe \spc, we discuss how we can extend the typical RL formulation to allow for the notion of \tasks.
RL problems are generally modeled as Markov Decision Processes (MDPs), i.e., a $4$-tuple $\mathcal{M} \coloneqq (S, A, T, R)$ consisting of a state space $S$, a set of actions $A$, a transition function $T: S \times A \to S$ and a reward function $R: S \times A \to \mathbb{R}$. 
This abstraction however, only allows to model a specific instantiation of a problem and does not allow to deviate from a single \emph{fixed} \task.

An \task $\inst\in\insts$ in a set of \tasks $\insts$ could, e.g., determine a different goal position in a maze problem or different gravity conditions (i.e., moon instead of earth) for a navigation task.
Information about the \task at hand is called its context $c_\inst$.
This context can either directly encode information about the \task, e.g., the true goal coordinates, or the kind of robot that should be controlled.

In order to make use of context in our problem description, we consider contextual MDPs~\citep[cMDP;][]{hallak-corr15,modi-alt18,biedenkapp-ecai20}.
A contextual MDP $\mathcal{M}_{\insts}$ is a collection of MDPs $\mathcal{M}_{\insts} \coloneqq \{\mathcal{M}_{\inst}\}_{\inst \in \insts}$ with $\mathcal{M}_{\inst} \coloneqq (S, A, T_i, R_i)$. As the underlying problem stays the same, we assume the possible state and action spaces are consistent across all \tasks; however, the transition and reward functions are unique to each \task.\footnote{In goal-directed RL, \tasks can also only vary the reward function and keep dynamics constant.}

An optimal policy $\policy^*$ for such a cMDP optimizes the expected return over all \tasks $\insts$ with discount factor $\gamma$:
\begin{equation}
    \policy^* \in \argmax_{\policy \in \Pi} \frac{1}{|\insts|} \sum_{\inst \in \insts} \sum_t^T \gamma^t R_i(s_t, \policy(s_t))
\end{equation}
As the reward depends on the given \task $\inst$, an agent solving a cMDP can leverage the context $c_\inst$ along with the current state $s_t\in S$ in order to differentiate between \tasks.

\section{Self-Paced Context Evaluation}\label{space}
In order to generate a curriculum without any prior knowledge of the target domain, our \emph{Self-Paced Context Evaluation (\spc)} takes advantage of the information contained in an agent's state value predictions. 
By modelling $V^\pi(\stateRL_t, c_\inst)$, the agent learns to predict the expected reward from state $\stateRL_{t}$ on \task $\inst$ when following the current policy~$\policy$.
Therefore, we propose $V^\pi(\stateRL_0, c_\inst)$ as an estimate of the total expected reward given a starting state $\stateRL_0$.\footnote{For simplicity's sake, we assume that an environment has a single starting state $s_0$ and we do not integrate over all possible starting states.}

\begin{definition}
The \emph{performance improvement capacity (PIC)} of an instance is the difference in value estimation between point $t$ and $t-1$, that is:
\begin{equation}
d_t(\inst) = V^\pi_t(\stateRL_0, c_\inst) - V^\pi_{t-1}(\stateRL_0, c_\inst).
\end{equation}
\end{definition}

The intuition is, if the instance evaluation changes by a large amount, the agent has learned a lot about this instance in the last iteration and can potentially learn even more on it. Instances that are too easy or too hard will yield relatively small or no improvements.
\spc prefers instances on which it expects to make most learning progress. As most state-of-the-art RL algorithms use a value-based critic, each instance's PIC is easily computed during training.

Algorithm~\ref{SPACE} summarizes the idea of $\spc$. After some initialization in Lines~1-3, \spc
performs an update step for the current policy $\pi$ and the value function $V^\pi$ based on roll-outs on the current instance set $\insts_{curr}$. In principle, any RL algorithm with a value-function estimate can be used, such as $\mathcal{Q}$-learning or policy search based on an actor-critic.
 In Lines~6-7, \spc updates the average instance evaluation and the difference to the last iteration; note that this only considers the current set of instances $\insts_{curr}$. 
In \mbox{Lines~8-9}, \spc first checks whether the value function $V_t^\pi$ changed $\Delta V_t^\pi$ by a factor \mbox{$\eta < 1.0$} compared to the value function before the update.
If the update led to an insignificant change of the value function, \spc assumes that the learning sufficiently converged and we can add $\kappa$ new instances to $\insts_{curr}$.
Starting in Line~10, \spc determines which instances in $\insts$ should be included in $\insts_{curr}$. For each instance, \spc first computes how much the value function changed, $d_t(i)$.
The instances with the highest PIC regarding $V^\pi$ are chosen as $\insts_{curr}$ (Lines~12-13), assuming that it is easy to make progress on these instances right now.
Note, we evaluate the influence of the $\eta$ and $\kappa$ hyperparameters on the learning behaviour of \spc in our experiments.

 \begin{algorithm2e}[tbp]
    \caption{\spc curriculum generation}
    \label{SPACE}
    {
    \KwData{policy $\policy$, value function $V$, \task set $\insts$, threshold $\eta$, step size $\kappa$, \#iterations $T$}
        $S, t := 0$\;
        $V_0 := 0$\;
        $\mathcal{I}_{curr} := \{i\}$ with $i$ randomly sampled from $\insts$\;
    	\For{$t=1 ... T$}{ 
    	    $\policy, V^\pi_t := update(\policy, V^\pi_{t-1}, \mathcal{I}_{curr})$\;
    	    $V_t^\pi := \frac{1}{|\mathcal{I}_{curr}|}\sum_{i \in \mathcal{I}_{curr}} |V^\pi_{t}(s_0, c_i)|$\;
	    	\If{$V_t^\pi \in [(1-\eta)V_{t-1}^\pi, (1+\eta)V_{t-1}^\pi]$}{
	    	\tcp*[h]{Increase set size}\;
	    	$S := S + \kappa$\;}
	    	\tcp*[h]{Choose next instance set}\;
	        \ForAll{$i \in \mathcal{I}$} 
		        {$ d_t(i) := V^\pi_t(\stateRL_0, c_\inst) - V^\pi_{t-1}(\stateRL_0, c_\inst)$}
	    	$\mathcal{I}_{curr} := S$ instances with highest $d_t(i)$\;
            $t := t+1$\;
    	}}
  \end{algorithm2e}

\subsection{Exemplary Application of \spc}

As a motivating example, we consider the CartPole environment \cite{gym} with three different pole lengths, see Figure~\ref{fig:cartpole_instances}.
We use a small DQN (hyperparameters given in Appendix~\ref{app:hypers}) for this example with the pole length being given as an additional state feature.
Although CartPole is generally considered as easy to solve, using poles of different length causes the DQN using a round robin curriculum to be unable to improve over time (see Figure~\ref{fig:cartpole_perf}).
\spc on the other hand is able to generate curricula that allow the DQN to learn how the cart has to be moved for the different poles and thus improve considerably to a mean performance of around $150$ per episode compared to round robin's $25$.
\begin{figure}
    \includegraphics[width=0.5\textwidth]{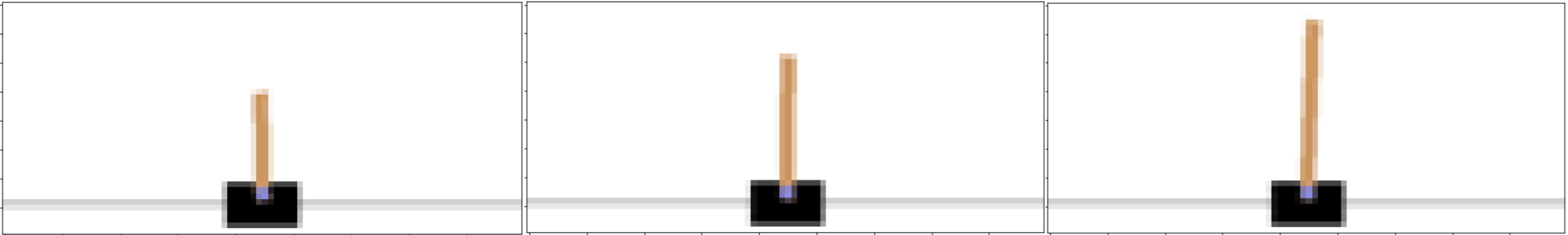}
    \caption{CartPole \tasks: short (s), medium (m) and long (l) balancing pole.}
    \label{fig:cartpole_instances}
\end{figure}
\begin{figure}
    \centering
    \includegraphics[width=0.5\textwidth]{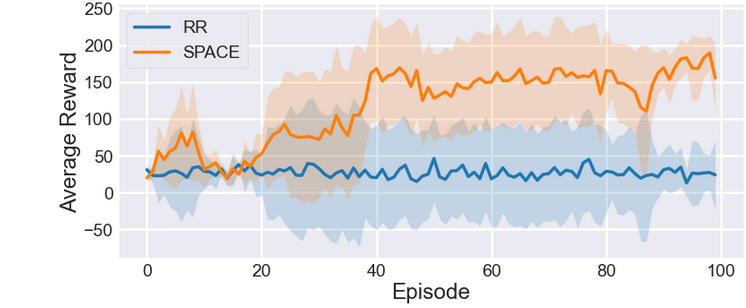}
    \caption{Performance ($\pm$ std.) comparison of \spc and default instance ordering on CartPole over $5$ seeds each.}
    \label{fig:cartpole_perf}
\end{figure}

In Figure~\ref{fig:cartpole_curriculum} we can see the main difference between the two methods. 
While the round robin agent trains on all three different variations one episode each, \spc only chooses to train on the cart with the long pole twice before episode $40$.
Instead, the focus is on a single instance at a time, using either the short or medium pole and changing not every episode but trains on an instance for at least three consecutive episodes. 
This shows that the value function can provide guidance as to instance similarity, as we would expect that the short and medium sticks behave in a similar way, as well as difficulty, the long pole being the hardest to control of the three.
While the changes in the value function may not provide a completely stable curriculum, training on one instance for a flexible amount of episodes instead of one episode already has a big impact on overall performance.
Furthermore, comparing the curriculum to the performance curve, focusing on only one instance at a time already leads to the agent performing considerably better on all of them. 
This validates the idea that there are underlying dynamics common to all three pole lengths which are important to learn and then refine according to the instance dynamics.
\begin{figure}
    \centering
    \includegraphics[width=0.5\textwidth]{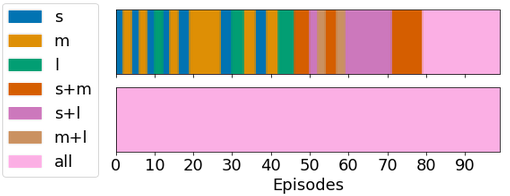}
    \caption{One exemplary run of \spc (top) and round robin (bottom) curricula on CartPole.}
    \label{fig:cartpole_curriculum}
\end{figure}

\subsection{Convergence of \spc}\label{sub:convergence}

To discuss under which conditions \spc will converge, we consider two cases.

\begin{theorem}\label{theorem:all_insts}
Given a set of \tasks $\mathcal{I}$ that are sufficiently distinguishable by their context $c_i$ as well as an instance of \spc with $\eta > 0$, $\kappa \geq 1$ and an agent with value function $V_t$. 
If the value function estimation $V^\pi_t$ converges in the limit to some value function $V$ on each instance (\mbox{$\forall i \in \mathcal{I}. \forall s \in S: \lim_{t \to \infty} V_t^\pi(s, c_i) = V(s, c_i)$}) and globally (\mbox{$\lim_{t \to \infty} V_t^\pi(s) = V(s)$}), \spc{} will eventually include all \tasks in the curriculum.
\end{theorem}
\begin{proof}
Since for all $\inst \in \insts$:
\begin{equation}\label{eq:single_convergence}
\lim_{t \to \infty} V_t^\pi(s_0, c_{\inst}) = V(s_0, c_{\inst})
\end{equation}
and therefore it follows that:
\begin{equation}
    \lim_{t \to \infty} \Delta V_{t-1}^\pi = ||V^\pi_t(s_0, c_i)| - |V^\pi_t(s_0, c_i)|| \to 0
\end{equation}
Thus \spc is guaranteed to include at least one other \task $\inst'$ in the new curriculum $\mathcal{I}_{curr}$ at some point $t$.
Now we assume that we are given any $\mathcal{I}' \subseteq \mathcal{I}$ with size \mbox{$n < |\mathcal{I}|$}.
As \mbox{$\forall i \in \mathcal{I}. \forall s \in S: \lim_{t \to \infty} V_t^\pi(s, c_i) = V(s, c_i)$} and Equation~\ref{eq:single_convergence}, convergence of $V_t^\pi$ on the subset $\mathcal{I}'$ follows:
\begin{equation}
    \forall i \in \mathcal{I}': \lim_{t \to \infty} V^\pi(s_0, c_i) = V(s_0, c_i)
\end{equation}
Therefore, as in the single instance case:
\begin{equation}
    \lim_{t \to \infty} \Delta V_t^\pi \to 0
\end{equation}
and a new instance is added.

As the curriculum is guaranteed to be extended for any instance set of size $n=1$ and $n\leq|\mathcal{I}|$, \spc will eventually construct a curriculum using the whole instance set.
\end{proof}

\begin{corollary}
If \spc{} covers all instances at some time point, it will be only slower than round robin by a constant factor in the worst case.
\end{corollary}
\begin{proof}
Assume $\kappa = 1$ and that the learning agent requires $\mathcal{O}(K)$ steps to converge on a single \task.

If the agent is not able to transfer any of its gained knowledge between any of the tasks, \spc will require to train an agent $\mathcal{O}(k)$ steps before growing the curriculum, where $k \leq K$, depending on $\eta$.
\spc will thus require $\mathcal{O}(|\mathcal{I}| \cdot k)$ steps to include all $|\mathcal{I}|$ \tasks in the curriculum.
At this point, \spc behaves as a round robin schedule does, i.e., iterating over each \task while training the agent.
Therefore, even if the construction of a meaningful curriculum should have failed, \spc{} can recover by falling back to a round robin scheme after $\mathcal{O}(|\mathcal{I}| \cdot k)$ steps.
\end{proof}

\begin{corollary}
If the value function estimation converges to the \emph{true} value function $V^*$, \spc will also converge to the optimal policy. \hfill\qedsymbol
\end{corollary}

Assume the worst case in which the value function estimate does not converge, but either oscillates or even diverges.
This could happen if \spc{} jumps between two disjoint instance sets $\mathcal{I}_1$ and $\mathcal{I}_2$ and the progress on $\mathcal{I}_1$ is lost by switching to $\mathcal{I}_2$ and vice versa.\footnote{Note: Though theoretically possible, we have never observed this problem in practice.}
Whenever we detect that learning is not progressing further and convergence is not achieved (i.e., $\Delta V^\pi_t \neq 0$ and $\mathcal{I}_{curr} \neq \mathcal{I}$),
\spc could simply increase $\eta$. As this hyperparameter controls how strict the convergence criterion is, increasing the value will allow for new instances to be added to the training set even though the original convergence criterion has not been met. The least value to  which $\eta$ should be set to guarantee an increase of instances is $\frac{\Delta V^\pi_t + \epsilon}{V^\pi_{t-1}}$ for any $\epsilon > 0$ to eventually train on all instances.

\begin{theorem}\label{theorem:dynamic_eta}
If the value function estimate is not guaranteed to converge (e.g., in deep reinforcement learning), \spc{} can still recover a round robin scheme by increasing the threshold $\eta$ if needed.
\end{theorem}
\begin{proof}
If at any point $t$, $\Delta V^\pi_t \neq 0$ and $\mathcal{I}_{curr} \neq \mathcal{I}$, we apply the method described above and set $\eta = \frac{\Delta V^\pi_t + \epsilon}{V^\pi_{t-1}}$.
Then the condition to increase the instance set size is $\Delta V_t^\pi < \eta \cdot V_{t-1}^\pi \rightarrow \Delta V_t^\pi < \frac{\Delta V^\pi_t + \epsilon}{V^\pi_{t-1}} \cdot V_{t-1}^\pi \rightarrow \Delta V_t^\pi < \Delta V_t^\pi + \epsilon$.
Thus the instance set size is guaranteed to be increased.
As this is true for any point in training, \spc{} can still consider all instances at some point $t^*$ and thus perform as well as round robin from $t^*$ onward. 
\end{proof}

\section{Experiments}\label{experiments}
In this section we empirically evaluate \spc on two different environments.
The code for all experiments is available at \href{https://github.com/automl/SPaCE}{https://github.com/automl/SPaCE}.
We first describe the experimental setup before comparing \spc against a round robin (RR) training scheme and SPDRL~\cite{klink-neurips20-spdrl} as a state-of-the-art self-paced RL baseline.
Finally we evaluate the influence of \spc's own hyperparameters and limitations.

\subsection{Setup}
We evaluated \spc in settings that readily allow for context information to encode different \tasks, namely the \emph{Ant} locomotion environment~\cite{pybullet}, the gym-maze environment~\cite{mazes} and the \emph{BallCatching} and contextual \emph{PointMass} environments as used by \citeN{klink-neurips20-spdrl}.

The task in Ant is to control a four legged ant robot towards a goal on a flat 2D surface as quick as possible.
The context is given by the x- and y-coordinates of the goal.
Goals that are close to the starting position are easier to reach and thus we expect them to be easier to learn and their policies to transfer to more difficult \tasks.
Additionally, the context indicates if no or up to one of the four legs of the ant robot is immobilized, similar to \cite{seo-neurips20}.
We uniformly sampled $200$ \tasks which we split in equal sized, disjoint training and test sets (see Appendix~\ref{app:instdist}).
The context of the maze environment~\cite{mazes}, in which the task is to find the goal state, is given as the flattened 5x5 layout of the current instance. $100$ training and test instances each were sampled using the given maze generator.
The agent's goal in BallCatching is to direct a robot to catch a ball. 
The ball's distance from the robot as well as it goal position are given as context information.
Training and test sets were each $100$ instances large and uniformly sampled between our context bounds.
In the PointMass environment (see Figure \ref{fig:example_instances}), an agent maneuvers a point mass through a goal in a two-dimensional space.
The goal position, the width as well as the friction coefficient of the ground are given as context.
We sampled $100$ \tasks for training and testing, each for two different distributions.
The first distribution is chosen to cover the space of possible \tasks, whereas the second distribution follows that of \citeN{klink-neurips20-spdrl} and focuses on an area around a particularly difficult \task (see Appendix \ref{app:instdist}).

To be consistent and fair with respect to prior work, we trained a PPO agent \cite{schulman-arxiv17a} for Ant and a TRPO agent for PointMass \cite{schulman-corr15} and base our curriculum generation on their value-based actors.
For easier readability, all plots are smoothed over $10$ steps.
In order to monitor generalization progress over time, we evaluated the agent on all \tasks in the training and test set after each complete run through the training set.
As the results on training and test sets were very similar, we only report the test performance. In all experiments we evaluated our agents over $10$ random seeds.
For hardware specifications and hyperparameters, please see Appendix \ref{app:hypers}.

\subsection{Baselines}
In our experiments, we use three different baselines to compare \spc's performance to.
\paragraph{Round Robin (RR)} To be sure \spc outperforms instances without an intentional ordering, we compare against round robin as a common default instance ordering. 
This means that the training instances are ordered in an arbitrary way and we simply iterate over them, playing one episode per instance.
As the instance sets we use are generated randomly, this ordering is chosen at random as well. 

\paragraph{SPDRL}SPDRL \cite{klink-neurips20-spdrl} is a state-of-the-art self-paced learning method for contextual RL. 
This is notable as most curriculum learning methods are explicitly designed for goal-directed RL, which makes them unsuitable in our setting.
Counter to \spc and RR, SPDRL makes use of \atask distribution to continually sample new \tasks of a specific difficulty level.
SPDRL uses this ability to generate new \tasks to focus on particularly difficult \tasks, while largely ignoring the remaining \task space.
To this end, SPDRL requires additional domain knowledge, besides the context information, to determine which \tasks SPDRL should focus on.
Therefore we provide SPDRL with the distribution of our training and test set to focus the learning on its center.

\begin{figure*}[tbp]
    \centering
    \includegraphics[width=0.45\textwidth]{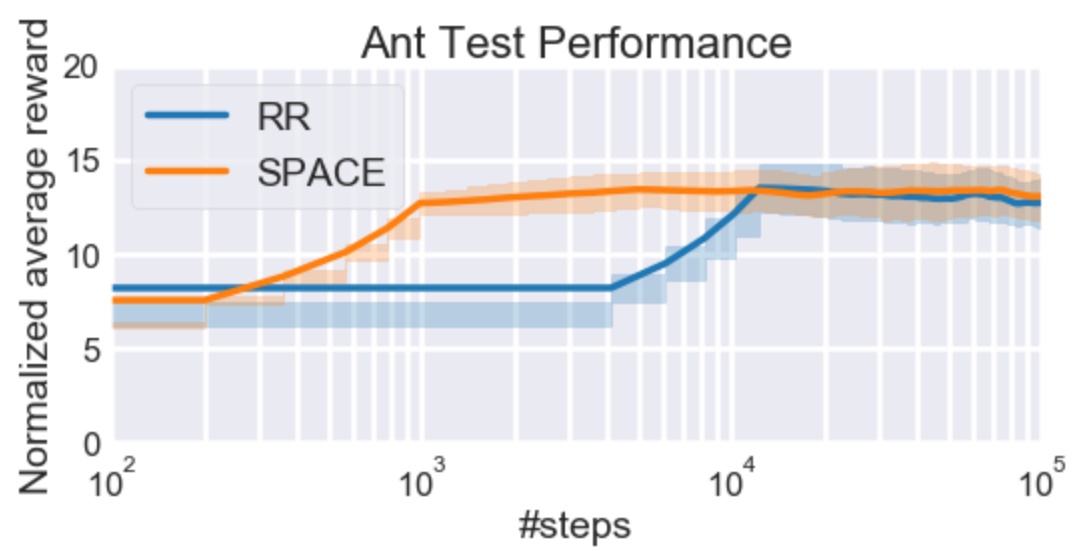}
    \includegraphics[width=0.45\textwidth]{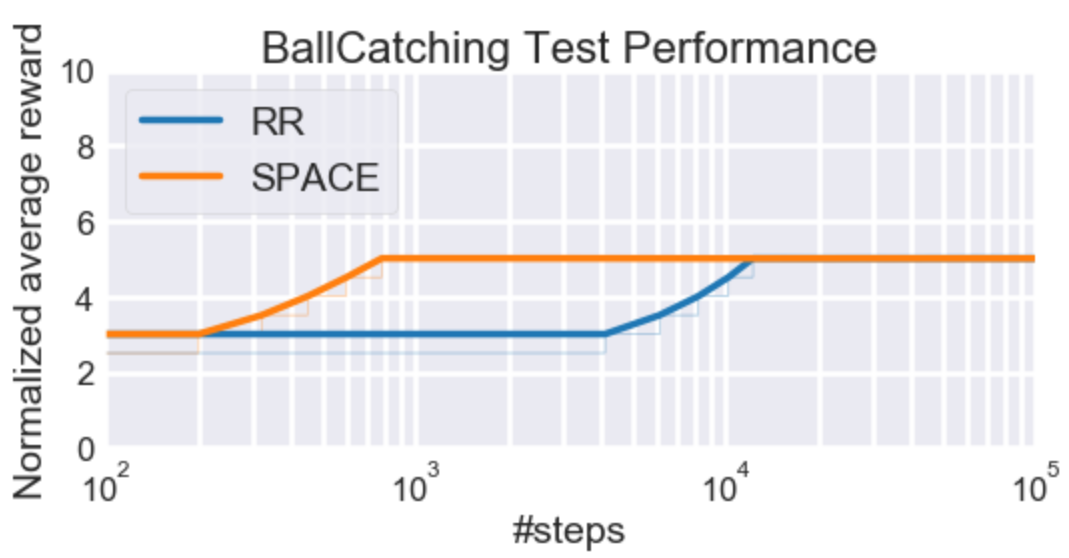}
    \caption{Mean reward ($\pm$ standard deviation) per episode over $10$ runs on Ant (left) and BallCatching (right).}
    \label{fig:ant_vs_rr}
\end{figure*}

\paragraph{c\spc}With \spc we opted for taking an agent's knowledge about the expected reward, i.e. value function, into account to determine the similarity and difficulty of \tasks.
However, as \tasks can be represented by their context, their similarity could also be quantified directly through their similarity in the context space.
This form of similarity quantification is common in fields making use of techniques such as algorithm selection~\cite{rice76a} and meta-learning~\cite{brazdil-08}.
Such curricula order \tasks according to their context similarities. A successful application of this approach can be seen in Reverse Curriculum Generation for Reinforcement Learning~\cite{florensa-corl17} where robot arm starting positions were ordered into a curriculum according to their similarity.
In other words, instead of using an agent's performance evaluations as a basis for the curriculum generation, instances with contexts that are closest to the current curriculum context are added. 
\spc's instance ordering criterion can easily be changed to compare context space distance instead of evaluations, yielding a variation we call \emph{context \spc} (c\spc).
More precisely, we replaced $d$ as our instance selection criterion (see Algorithm~\ref{SPACE} Line $10$) with the Euclidean distance to the current instance set $\mathcal{I}_{curr}$.
In such cases, c\spc can suffer from the same problems as unsupervised learning.
A priori it is not clear how to scale and weight the different context features without having any signal how the features will affect the difficulty of instances and how good the resulting curriculum will be. 
In contrast to \spc, we deem this a potential challenge in applying c\spc.
For this reason, we recommend \spc as the default approach whenever state evaluations are available.

\subsection{Does the Instance Order Matter?}
We first compare \spc to a baseline round robin (RR) agent on the Ant and BallCatching environments, to determine if \spc can find a curriculum that outperforms a random ordering.
In Figure \ref{fig:ant_vs_rr}, both agents reach the same final performance in each environment, but the agent trained via \spc learns considerably faster.
It only requires $10^3$ steps to reach a reward of around $11$ in Ant whereas RR requires roughly $10\times$ as many steps to train an agent to reach the same reward.
The results for BallCatching are similar, with \spc again being faster to reach the final performance a factor of at least $10$.
We further compare both methods on the PointMass environment when training on an \task set that was uniformly sampled from the space of possible \tasks (see Figure~\ref{fig:test_spdl}).

Here, the agent trained with \spc is only roughly twice as fast, but it substantially outperforms round robin in terms of final performance.
As the RR baseline does not care about the order in which \tasks are presented to the agent, we conclude that a more structured learning approach is needed.
From \spc's performance we can conclude that a curriculum, learned in a self-paced fashion can help improve both training performance and generalization.
The experiments in the following sections further confirm this finding.

\subsection{Comparing \spc and SPDRL}
\begin{figure}[tbp]
    \centering
    \includegraphics[width=0.45\textwidth]{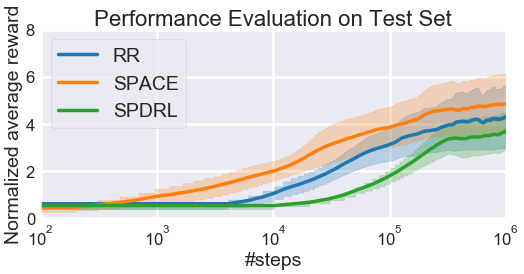}
    \caption{Mean reward per episode on a test set of uniformly sampled \tasks for PointMass.}
    \label{fig:test_spdl}
\end{figure}
We further compared \spc to SPDRL \cite{klink-neurips20-spdrl} on the PointMass environment in order to demonstrate the difference between \spc and an other self-paced learning method.
We used the same implementation and hyperparameters as in \citep{klink-neurips20-spdrl} for SPDRL.
The test performance of the agents can be seen in Figure~\ref{fig:test_spdl}.

As \mbox{SPDRL} was developed to train an agent to solve specific hard \tasks in the PointMass environment, it clearly falls short when it comes to covering the whole \task space.
The agent trained via SPDRL learns much slower and achieves a worse final performance than an agent trained via RR.
Perhaps unsurprisingly, this shows that targeting learning on hard instances does not imply the same agent can achieve good generalization performance on all instances.

\subsection{How Well does \spc Handle Complex Contexts?}

\begin{figure}
    \centering
    \includegraphics[width=0.45\textwidth]{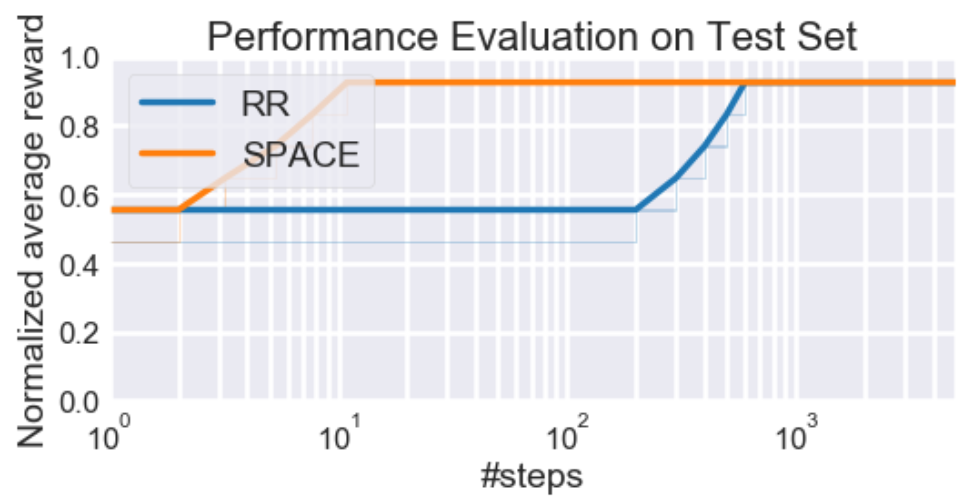}
    \caption{\spc and RR of a set of mazes over 10 runs ($\pm$ standard deviation)}
    \label{fig:mazes}
\end{figure}
While our benchmarks above are common meta-RL problem settings with different complexities, their contexts are given in a rather simple form, i.e., a short context vector directly describes the goal and environment dynamics.
The agent can therefore make a direct connection between the changes between instances and the different context descriptions.
This may not always be the case with context possibly being given within the observation, e.g., as part of an image.

In order to confirm that such a context description still enables \spc to select the appropriate next instance, we use a set of $100$ 5x5 mazes \cite{mazes} for our agent to generalize over.
The observation is the agent's current position while the context is given by the flattened maze layout. 

This context is much more complex than the previously used ones by having a structure that has been flattened and its components do not directly correlate to an increase in difficulty.
Furthermore, many of the components of the context may not change from instance to instance even though the layout, and therefore the required policy, will.

Figure~\ref{fig:mazes} shows that while the context information for this task is much more complex than previously seen, \spc still outperforms the round robin agent in a similar way than it does to for Ant and BallCatching.
The round robin agent needs several hundred episodes to solve all mazes while \spc is able to generalize from just $10$ episodes.
While the context complexity increases in this case, the value function is still able to differentiate between them enough to allow a distinction between different instances.
Therefore we expect that the representation of the context is not a major concern for the performance of \spc.

\subsection{Can \spc Be Applied Without a Value Function?}
The PointMass environment has three different context features for which we can easily use the context space distance to construct a curriculum. 
\spc and c\spc perform similarly on PointMass (Figure~\ref{fig:test_plus_cl}) in terms of learning speed and overall performance, both reaching the same performance at the same speed, with \spc learning faster on average between $10^4$ and $10^5$.
This makes both \spc variations a better choice than round robin.
\begin{figure}[tbp]
    \centering
    \includegraphics[width=0.45\textwidth]{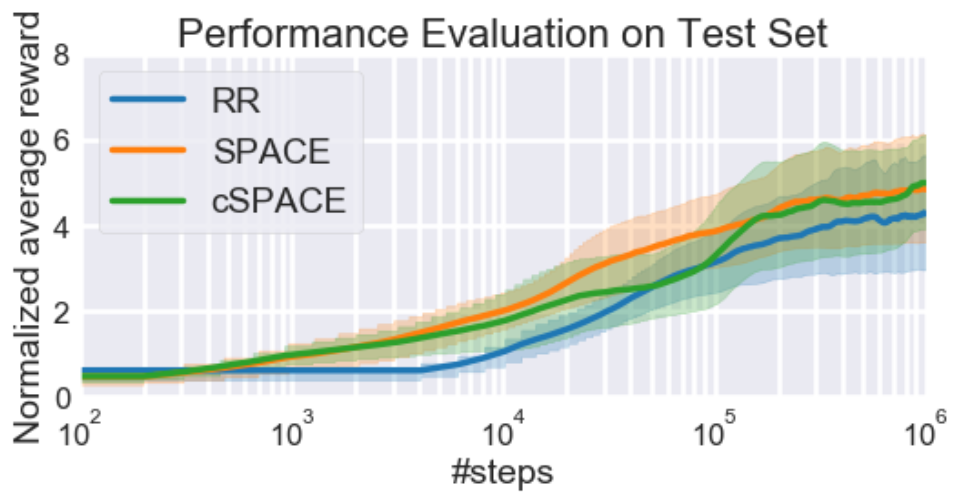}
    \caption{Mean reward on contextual PointMass with additional c\spc results.}
    \label{fig:test_plus_cl}
\end{figure}

Both c\spc and \spc are also consistent in the curricula they find.
We measured this by comparing the frequency with which each \task was used in the training set compared to the weighted frequency which gives higher rank to \tasks chosen at earlier iterations ($1$ for the \task chosen first down to $\frac{1}{|\mathcal{I}|}$ for the \task chosen last).
For c\spc, both the frequency and weighted frequency stayed the same while for \spc only the four least used \tasks differed in order between the two. 

We also compared the mean \task distance between curriculum iterations to see which method allows for smoother transitions between tasks.
Smooth transitions correlate to a handcrafted curriculum where \tasks are close together in the context space, making the curriculum easier to learn from a human perspective.
\spc moves the \task set around $4.7\%$ each curriculum iteration while c\spc moves by around $5.6\%$.
The maximum induced change is $10.1\%$ for \spc and $13.3\%$ for c\spc approach.

As we can see from these comparisons, using the information contained in an agent's value function to construct a curriculum is very similar to using the context space distance.
It needs to be said, however, that in PointMass the context reflects the environment dynamics in a very direct way, being made up of the x- and y-positions of the goal to reach and the friction coefficient.
Therefore we would expect c\spc to perform very well on such environments. 
The fact that the default \spc setting performs similarly indicates that the value function contains the information necessary to order instances into a curriculum of similar quality.
As not all environments may have such simple changes between instances, we expect that c\spc has limitations on those kinds of environments while we can expect the value-based \spc variation to continue constructing high quality curricula even in that case.

\subsection{How Robust is \spc wrt its Hyperparameters?}
\begin{table}
        \caption{Mean reward $\pm$ standard deviation for different hyperparameter values on PointMass after $10^6$ steps.}\label{fig:ablation}
        \centering
        \begin{tabular}{ccccc}
        \toprule
        {} & \multicolumn{4}{c}{$\eta$}\\
        \cmidrule(lr){2-5}
        $\kappa$ & $5\%$ & $10\%$ & $20\%$ & $40\%$ \\
        \midrule
        $1$ & $5.1 \pm 0.7$ & $4.8 \pm 1.2$ & $4.7 \pm 1.2$ & $5.2 \pm 0.7$ \\
        $4$ & $5.5 \pm 0.5$ & $5.2 \pm 0.7$ & $4.3 \pm 1.2$ & $4.4 \pm 1.0$ \\
        $16$ & $4.6 \pm 1.1$ & $3.7 \pm 1.1$ & $5.1 \pm 1.2$ & $4.8 \pm 1.0$ \\
        $32$ & $4.5 \pm 1.1$ & $4.6 \pm 1.1$ & $4.7 \pm 1.3$ & $5.0 \pm 1.2$ \\
        \bottomrule
        \end{tabular}
\end{table} 
\spc comes with two hyperparameters, the performance threshold for curriculum interactions $\eta$ and the \task increment $\kappa$.
These hyperparameters interact with each other to make \spc comparatively stable across different hyperparameter values (as seen in Figure \ref{fig:ablation}). 

By varying $\eta$ for a given value of $\kappa$, we alter the degree of stability the agent's value estimates have to reach between training episodes. 
Depending on the problem at hand, the value estimates may never be perfectly stable, therefore a very low value for $\eta$ may prevent the training set from expanding.
On the other hand, a very large value will move \spc closer to round robin.
Thus we view $\eta$ as the more important hyperparameter of the two.

Our study shows  very little performance differences for different values of $\kappa$ and $\eta$.
In part, this is because PointMass \tasks are not too difficult in the mean, therefore adding many at once does not heavily disturb learning. Larger performance thresholds $\eta$ are not an issue for this reason. A value of $5\%$ for $\eta$ seems quite low, but as the instances are relatively easy, the agent can still converge enough very quickly.
\begin{figure}[tbp]
    \centering
    \includegraphics[width=0.45\textwidth]{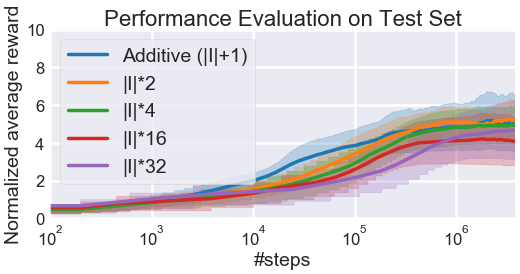}
    \caption{Mean reward per episode on a test set with fast rising \task set size (i.e. varying $\kappa$) and fixed $\eta=10\%$.}
    \label{fig:multikappa}
\end{figure}
Different values for $\kappa$ show similar results here. We expect this hyperparameter to be more important in very diverse settings with large gaps between \tasks.
We can see the effect if we multiply the size of our training set instead of adding \tasks (see Figure \ref{fig:multikappa}). In this case, there is a visible slowdown, supporting that $\kappa$ has a big influence on training performance.

From these results, we believe that it is reasonable to recommend keeping $\kappa=1$ for most applications. 
It can yield more fine-grained curricula which will be important on diverse \task sets and will likely only impact training on very large \task sets.
For $\eta$, using a low value such as $5\%$ should ensure that the agent will not be overwhelmed with new \tasks if it takes more than one curriculum iteration to learn from the current training set.

\subsection{Can Goal-Directed RL Achieve Similar Results?}\label{goalRL}
\begin{figure}[tb]
    \centering
    \includegraphics[width=0.45\textwidth]{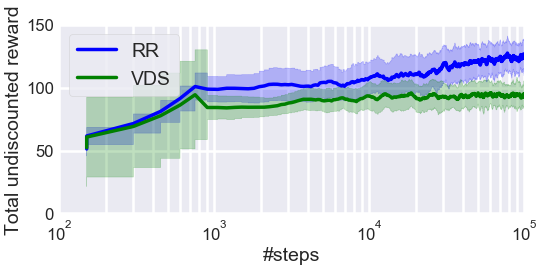}
    \caption{Total undiscounted reward of VDS and its RR baseline on AntGoal.}
    \label{fig:vds}
\end{figure}
Goal-directed RL and contextual RL are closely related flavours of RL. 
The main difference between the two is that in goal-directed RL, the context is restricted to only contain information about a desired goal-location. 
Thereby the context indicates a (final) state in which the agent should reach.
Policies learning with goals as context information thus can learn the value of executing an action in a certain state for reaching the desired goal.
Crucially however, in this setting transition dynamics are assumed to never change for different contexts.

Contextual RL subsumes goal-directed RL by also allowing the environment dynamics for the same goal to vary. 
Both environments we evaluate on exemplify this. In contextual PointMass, we specify both a goal and the friction coefficient to describe an instance. 
Even if the goal is kept the same, however, the friction is a deciding factor in the amount of force that is necessary to reach the goal correctly.
So while the force direction is the same for instances with the same goal, the required amount of force depends on the friction and therefore the agent needs both information to learn to generalize across different \tasks.

On AntGoal, we can see how this looks in practice. VDS \cite{zhang-neurips20-vds} is a recent state-of-the-art method for goal-directed curriculum construction based on HER \cite{andrychowicz-neurips17-her}. 
As a result, it can take only the ant's goal into consideration when selecting the next instance and crucially misses that in different instances the ant has different defects in some of its joints. 
As a result, the method conflates all instances with the same goals and fails to actually learn how to act on any of them.

We used the implementation and baseline of \citet{zhang-neurips20-vds} to demonstrate that while goal-directed curriculum generation approaches seem similar to SPaCE, our problem setting is out of scope for them (see Figure~\ref{fig:vds}). 
Their RR baseline has a different learning curve as ours as the algorithm used is different, but it clearly is able to improve over time.
As VDS uses goals to describe the necessary behaviour, it cannot do the same.
Therefore, curriculum learning methods for contextual RL and goal-directed RL have different scopes and cannot be compared fairly in the contextual RL setting.

\section{Limitations}
Even though \spc performed very well on the benchmarks used in this paper, there are several limitations of \spc to be considered.
The first one is that the problem and the \task set both need to support curriculum learning to some degree.
For the problem itself this means that the policy to solve it is influenced by context to a large degree, but that there is an underlying structure that can be exploited using a curriculum.
The \task set then needs to be large enough to actually give \spc the opportunity to do so. In settings with too little or very large amounts of instances, \spc becomes less efficient (see Appendix~\ref{app:sizes}). 

Furthermore, if the \task set is very homogeneous, similar to the specific \task SPDRL uses on PointMass (see Appendix~\ref{app:SPDRL}), using different \tasks for training might not make a difference.
Conversely, if the \task set is heterogeneous, preliminary experiments showed that \spc requires a larger amount of \tasks to speed up the learning. Thus not every problem is suited for curriculum learning.

Lastly, \spc is constructed to work for discrete \task spaces only, where the \task ordering is essential for learning efficiency.
We stress the fact that \spc is designed for use cases with only a few \task examples. In settings with instance generators or a lot of domain knowledge available, it is likely better to exploit them which \spc is not designed for.

\section{Conclusion}
Self-Paced Context Evaluation (\spc) provides an adaptive curriculum learning method for problem settings constrained to a fixed set of training \tasks. 
Thereby we facilitate generalization in practical applications of RL. 
We demonstrated that the order of \tasks on which agents learn their behaviour policies indeed is important and can produce a better learning efficiency.
In addition, \spc outperformed a simple round robin baseline as well as more specialized curriculum learning methods requiring access to unlimited \task generators to perform well.
Finally we evaluated the influence of \spc's own hyperparameters and showed that they are robust on the chosen environments.

Future research could address how to derive performance expectations for practical applications of RL with a limited amount of \tasks with respect to the amount of information available.
Furthermore, we might be able to use value estimation to further improve training efficiency for example by clustering \tasks of similar difficulty and limiting the amount of training on very easy ones to a minimum. 
Another important factor for contextual RL in general is catastrophic forgetting (see~Appendix \ref{app:forgetting}), which is not yet sufficiently understood, especially in the continuous context spaces we applied \spc to.

\section{Acknowledgements}
Theresa Eimer and Marius Lindauer acknowledge funding by the German Research Foundation (DFG) under LI 2801/4-1.
All authors acknowledge funding by the Robert Bosch GmbH.

\bibliography{bib/strings,bib/lib,bib/spl,bib/proc}  %
\bibliographystyle{icml2021}

\appendix

\section{Instance Sampling}\label{app:instdist}
\paragraph{AntGoal}
We uniformly sampled $100$ different goals at a distance of at most $750$ in both x- and y-direction for both training and test set respectively.

\paragraph{BallCatching} 
The distance and goal coordinates were sampled uniformly for both training and test set. The distance ranged between $0.125 \cdot \pi$ and $0.5 \cdot \pi$, the x-coordinate between $0.6$ and $1.1$ and the y-coordinate between $0.75$ and $4.0$.
Each instance set contains $100$ instances.

\paragraph{PointMass}
For PointMass, we sampled two different \task sets. First, we used the context bounds of [-4, 4] for the goal position, [0.5, 8] for the goal width and [0, 4] for friction to uniformly sample \tasks. The goal was to cover the \task space as well as possible. Our second \task set was sampled using the target distribution of SPDRL, which are normal distributions for each context component with means 2.5, 0.5 and 0 respectively as well as standard deviations of 0.004, 0.00375, and 0.002.

\section{Experiment Hardware \& Hyperparameters}\label{app:hypers}
\paragraph{Hardware}
All experiments with \spc and the baseline round robin agent were conducted on a slurm CPU cluster (see Table \ref{tab:cpu_cluster}). The upper memory limit for these experiments was 1GB per run. The SPDRL experiments were replicated on a slurm GPU cluster consisting of 6 nodes with eight RTX 2080 Ti each. Here maximum memory was 10GB. Slurm scripts for the experiments on PointMass and Ant are provided in the supplementary material. Gridworld experiments are every small and can therefore be found in a jupyter notebook.

\begin{table}[h]
    \centering
    \begin{tabular}{c|c|c|c}
         Machine no. &  CPU model & cores & RAM \\
         \hline
          1 &  Xeon E5-2670 & 16 & 188 GB \\
          2 & Xeon E5-2680 v3 & 24 & 251 \\
          3-6 & Xeon E5-2690 v2 & 20 & 125 GB \\
          7-10 &  Xeon Gold 5120 & 28 & 187 \\
    \end{tabular}
    \caption{CPU cluster used for training}
    \label{tab:cpu_cluster}
\end{table}

\paragraph{CartPole}\label{app:grid}
We used a DQN implementation in the top-$10$ on the environment leaderboard to ensure fair performance for round robin and \spc agents \cite{cartpole-impl}. 
We did not change any hyperparameters from that implementation and used $\kappa = 1$ and $\eta = 2.5\%$ for all experiments.

\paragraph{Other benchmarks}\label{app:point}
For both experiments we used stable baselines version 2.9.0 \cite{stable-baselines} with TRPO for PointMass and PPO2 for all other benchmarks. The policies are encoded by an MLP in both cases, with two layers of $64$ units for PPO. For PointMass, we used the default from the SDPRL paper with $21$ layers of 64 units each. The discount factor was $0.95$. 
The PPO2 specfic hyperparameters included no gradient clipping, a GAE hyperparameter $\lambda$ value of $0.99$ and an entropy coefficient of $0$.
For TRPO we used again used the same hyperparameters as SPDRL with a GAE hyperparameter $\lambda$ of $0.99$, a maximum KL-Divergence of $0.004$ and value function step size of around $0.24$. 
Any hyperparameters not mentioned were left at the stable baselines' default values. The random seeds were used to seed the environments with the corresponding seeding method. 

\section{Additional Comparison to SPDRL}\label{app:SPDRL}
\begin{figure}[h]
    \centering
    \includegraphics[scale=0.25]{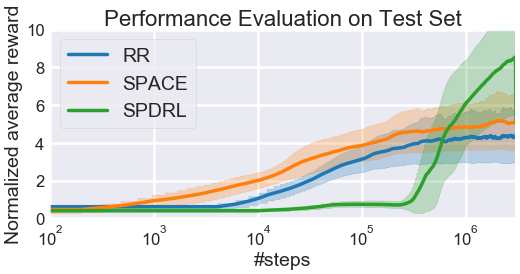}
    \caption{Mean reward per episode on a test set of hard \tasks with small goals and low friction.}
    \label{app-fig:test_spdl_hard}
\end{figure}
In contrast to \spc, SPDRL is designed to solve hard instances. To this end, it samples harder and harder instances over time. Therefore, we additionally study how \spc, round robin (RR) and SPDRL compare on hard instances sampled from the SPDRL target distribution, see Figure~\ref{app-fig:test_spdl_hard}.
\tasks in this distribution typically have small goal sizes and low friction, both of which contribute significantly to an increased difficulty.

As in the original paper, SPDRL was allowed to sample as many instances as needed from the distribution, whereas \spc and RR still only got access to a finite set of $100$ instances. 
In this setting, agents trained either via \spc or RR exhibit a similar learning behaviour as on the space covering \task set.
For the first $\sim200\,000$ steps both agents outperform the agent trained via SPDRL; RR anyway focuses on the whole target distribution from the beginnig and \spc is more free in the way it can select instances with fast training progress.
During this time, SPDRL trains the agents on some easy instances, while gradually adapting the \task distribution to focus on ever more difficult tasks.
Note that the level of difficulty is not determined solely by the agent being trained via SPDRL, as done in \spc, but is determined by an expert beforehand.

Once the agent trained via SPDRL is capable of homing in on the difficult \tasks it outperforms the other agents, as it can exploit its domain knowledge to sample ever more similarly difficult \tasks, while \spc and RR are stuck with the limited number of example \tasks and still try to cover the entire instance space.
To achieve this feat, SPDRL requires substantial expert knowledge about which \tasks to focus on.
In essence, the agent trained via SPDRL in the end is only capable of solving a few hard \tasks with very little variation and will fail to perform well on \tasks that are not narrowly aligned with the assumed \task distribution.

To be able to know which \tasks SPDRL should focus on, additional time and effort have to be spent to identify how to quantify \emph{difficulty} for SPDRL.
This effort is not reflected in Figure~\ref{app-fig:test_spdl_hard} and would move the curve of SPDRL even further to the right.

\section{Does the Training Set Size Matter?}\label{app:sizes}
To answer this question, we used \spc to train agents with varying \task set sizes. 
Figure \ref{fig:sizes} shows the test performance for differently sized \task sets. Intuitively, one might think that performance should improve with more \tasks as they cover the \task space better. 
Indeed, the results for training sets with only $25$ and $50$ \tasks are visibly worse than for larger sets. 
On the remaining \task sets, the agent show very similar performance, however. 
Note that the performance seems to increase from \atask set size of $100$ to $200$, but slightly drops again afterwards. 
There are multiple factors potentially contributing to this effect.

\begin{figure}[h]
    \centering
    \includegraphics[scale=0.25]{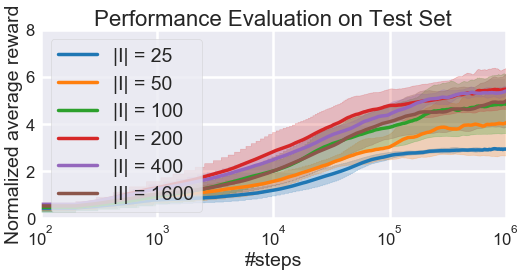}
    \caption{Mean reward per episode on test set for different sized \task sets.}
    \label{fig:sizes}
\end{figure}

The first is that the agent cannot incorporate any more information from the additional \tasks, maybe due to limited network capacity or due to the fact that smaller \task sets already cover the space adequately. 
Furthermore, as we only extend the \task set by one \task at a time, there are more learning steps between curriculum iterations the larger the \task set is, thereby slowing the process down.
Especially an agent trained on $1\,600$ \tasks will suffer from this.

Lastly, \spc improves upon the RR baseline by ordering training \tasks and thus smoothing the progression through the \task space. 
Larger \task sets offer an inherently smoother representation of the \task distribution, therefore diminishing the effect of \spc. 
In real-world application settings, we will rarely have access to such large numbers of instances and therefore, it is unlikely that such diminishing performance effects can be observed.
This  shows that the strength of our method comes to full effect when learning on a sparse representation of our \task space.

\section{Comparison of \spc Curricula}
To give some insight into which curricula \spc found on our benchmark environments, we compare how they behave across random seeds and how they compare to c\spc curricula.
We use Kendall's tau to determine how similar the order in which the instances are added to the training set is.

On PointMass, \spc finds a curriculum that stay very consistent across all random seeds, showing a correlation of at least $98.9\%$ each to the mean curriculum.
The same is true for the c\spc variation, where the correlation is above $93.8\%$ per seed.
Interestingly, these curricula are uncorrelated with a correlation of $-0.04$.
In both we cannot make out a human readable progression in a single context feature (see Figure~\ref{app-fig:pm_context}), their curricula do not correspond to any manual instance ordering.
As both perform well nonetheless, we can see that learning can be improved by multiple different curricula on this environment.
\begin{figure}
    \centering
    \includegraphics[width=0.5\textwidth]{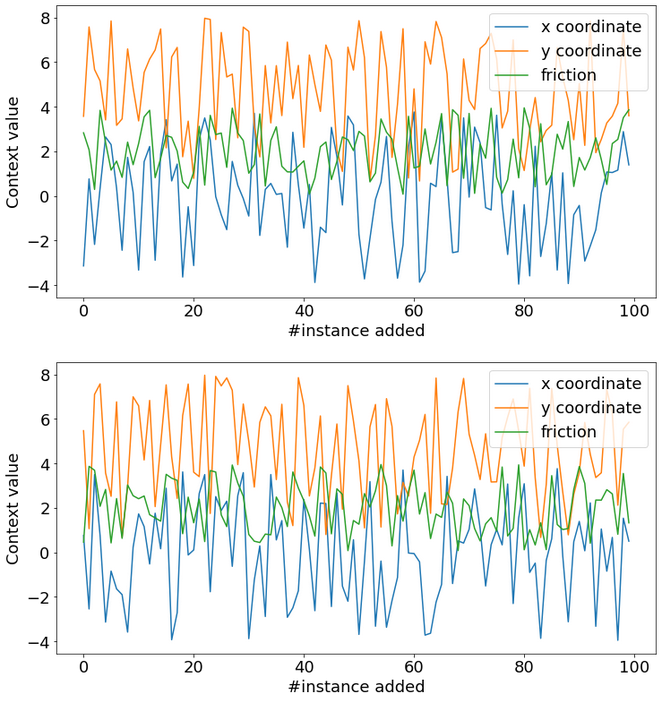}
    \caption{Context feature progression during training for \spc curriculum (top) and c\spc curriculum (bottom).}
    \label{app-fig:pm_context}
\end{figure}

\spc and c\spc produce almost equally unrelated curricula on AntGoal (correlation of $0.07$), but while the curriculum stays as consistent across seeds for c\spc, the same cannot be said for \spc. 
Here the correlation to the average curriculum ranges from $14.1\%$ to $52.4\%$. 
The correlations between the seed curricula fall into the same range, confirming that the \spc agent trains on a very different curriculum for each seed.
CartPole shows a similar behaviour, the curriculum varying quite a bit between seeds. 
Therefore we can conclude that \spc does not find a singular curriculum, but depends on the initialization of environment and model.
This is in contrast to c\spc which stays relatively static due to the context features being constant.

These comparisons suggest that we neither \spc nor c\spc finds an optimal curriculum for PointMass, AntGoal or CartPole. 
It seems, however, that we do not need an optimal curriculum for training at all, as even the $10$ very different curricula \spc finds on AntGoal perform vastly superior to the round robin default.
Curriculum Learning should thus focus on reliably and quickly finding good curricula in addition to finding qualitatively better ones.

\section{The Influence of Catastrophic Forgetting}
\label{app:forgetting}
When training across multiple instances, forgetting already learnt policies on a subset of instances is a concern \cite{beaulieu-ecai20}.
We analyze how often \spc and RR agents forget policy components in our PointMass experiments by observing performance development during training.
We selected PointMass for this analysis as here policies that are diverse both in how they react to different goal settings and different friction levels are required.
That means the policy has to completely change between the extremes of the context which is not required of our other benchmarks where underlying mechanics, e.g. walking for the Ant, stay very similar.

During the training on PointMass, we observed $8$ out of $100$ instances for which the performance decays after an initial improvement. 
We would expect the performance to stay at least constant if no forgetting takes place, so the agent likely forgets parts of the policy for these instances in favor of improving on others. 
The effect is about the same size for round robin agents where we can observe the same for $6$ out of $100$ instances.

Another reason for attributing this performance decay to forgetting is that on a purely goal-based PointMass variation, the number of instances on which we can observe this effect is slightly smaller (only $4$ instances), though not significantly so. All performance decay happens after learning has stagnated on all instances, however.
In this easier, purely goal-based setting we could therefore stop training early and would avoid performance decay entirely.
This points towards the added complexity of the setting being harder to capture for our agents.

While the effects on both \spc and RR agents are not very large in our experiments, catastrophic forgetting is therefore certainly important in the field of contextual RL. 
Future work could on integrate \spc with existing efforts to reduce this effect like ANML \cite{beaulieu-ecai20}.
A specific aspect of this research that would need to be extended is preventing forgetting in continuous context spaces in addition to the existing successes in discrete ones. 
\end{document}